\documentclass{sig-alternate}
\conferenceinfo{KDD 2011}{San Diego, CA, USA}
\usepackage{graphicx,color}
\usepackage{algorithm,algorithmic}
\usepackage{subfigure}
\usepackage{verbatim}
\usepackage{url}

\newtheorem{proposition}{Proposition}

\newtheorem{remark}{Remark}

\newdef{definition}{Definition}

\begin{document}

\title{Clustered Regression with Unknown Clusters}
\numberofauthors{2}
\author{
\alignauthor
Kishor Barman\\
       \affaddr{School of Technology and Computer Science}\\
       \affaddr{Tata Institute of Fundamental Research}\\
       \affaddr{Mumbai, India}\\
       \email{kishor@tifr.res.in}
\alignauthor
Onkar Dabeer\\
       \affaddr{School of Technology and Computer Science}\\
       \affaddr{Tata Institute of Fundamental Research}\\
       \affaddr{Mumbai, India}\\
       \email{onkar@tcs.tifr.res.in}
}


\date{}

\maketitle
\begin{abstract}
  We consider a collection of prediction experiments, which are clustered in the sense that groups of experiments exhibit similar relationship between the predictor and response variables. The experiment clusters as well as the regression relationships are unknown. The regression relationships define the experiment clusters, and in general, the predictor and response variables may not exhibit any clustering. We call this prediction problem clustered regression with unknown clusters (CRUC) and in this paper we focus on linear regression. We study and compare several methods for CRUC, demonstrate their applicability to the Yahoo Learning-to-rank Challenge (YLRC) dataset, and investigate an associated mathematical model. CRUC is at the crossroads of many prior works and we study several prediction algorithms with diverse origins: an adaptation of the expectation-maximization  algorithm, an approach inspired by K-means clustering, the singular value thresholding  approach to matrix rank minimization under quadratic constraints, an adaptation of the Curds and Whey method in multiple regression, and a local regression (LoR) scheme reminiscent of neighborhood methods in collaborative filtering. Based on empirical evaluation on the YLRC dataset as well as simulated data, we identify the LoR method as a good practical choice: it yields best or near-best prediction performance at a reasonable computational load, and it is less sensitive to the choice of the algorithm parameter. We also provide some analysis of the LoR method for an associated mathematical model, which sheds light on optimal parameter choice and prediction performance.

\end{abstract}


\section{Introduction}

Regression, which estimates the relationship between response variables and predictor variables, has a rich history (see for example \cite{Hastie,wasserman}). It is employed in a variety of fields such as machine learning, signal processing, etc. Often the response variables and predictor variables are collected from different experiments, and while the data may be different across experiments, there may be reason to believe that several experiments share the same regression relationship (such as the same regression parameter vector in case of linear regression). For example, in the Yahoo Learning-to-rank Challenge (YLRC) dataset \cite{ylrc}, there are several queries, and for each query, there are multiple URLs with relevance scores. This can be thought of as a regression problem between the relevance scores as the response variables and the feature vectors of the query-URL pair as the predictor variables. 
Since we expect many queries to be related to each other, we can postulate clusters of queries, with common regression parameters within a cluster, and different parameters across clusters. 
In this paper, we consider such a regression problem where the experiment
clusters are {\it unknown}. The term ``clustered regression" has already been used to refer to the case
when the clusters are known \cite{Liang_Zeger_1, Lloyd_1}. Hence, we refer to our problem as clustered regression with unknown clusters (CRUC).
We study several approaches to this problem arising from different perspectives. For our study, we use the YLRC dataset\footnote{While the focus of YLRC is on ranking URLs, our focus is on an associated regression problem. Our primary aim to show that the CRUC framework is meaningful for a real world dataset such as the YLRC dataset and we do not present any performance study of the ranking aspect.} as well as a mathematical model with corresponding analysis and simulations. In the remainder of this section, we briefly describe the algorithms we propose, their relationship to prior literature, and we outline our main results.

In Section \ref{sec:model}, we describe our basic setup in detail. In brief, we consider $M$ experiments with associated predictor and response variables, and we focus on linear regression. For this prediction problem, two methods immediately come to mind: a common linear regression fit across all experiments and an individual linear regression for each experiment. If we expect that several but not all of the experiments share the same regression relationship, then we wish for methods that lie between these two extremes. In Section \ref{sec:approaches}, we propose several such approaches and below we briefly outline them.

\noindent {\bf The EM Algorithm:} We can postulate that there are only $ 1 \leq K_0 \ll M$ distinct regression vectors (that is, we have $K_0$ clusters of experiments). We can treat the cluster index of each experiment as missing data, and then use the EM algorithm \cite{dempster} to iteratively compute an estimate of the cluster indices and the regression vectors. \\
\noindent {\bf K-means (KM) algorithm:} The classical K-means clustering algorithm \cite[Section 9.1]{bishop1} iteratively computes the cluster indices and the cluster centroids. In our context, we replace the centroid computation with least-squares regression estimation.\\
\noindent {\bf Singular value thresholding (SVT):} If we consider the $d \times M$ matrix of regression vectors, then for $K_0 < d$, it has a small rank. Hence another formulation is to minimize the rank of the regression vector matrix, subject to the condition that the corresponding mean-square error is small (a quadratic constraint). Such an optimization problem has been considered in \cite{Candes_cai} and we adapt their SVT algorithm to our problem.  \\
\noindent {\bf Curds and Whey (CW)}: In {\it multiple regression} \cite{breiman1}, the same predictor variables are used to predict multiple response variables. For $M=K_0$ and the same predictor variables across the experiments, our model reduces to that of multiple regression. In other words, our setup is a generalization of multiple regression, and we can modify the CW method from \cite{breiman1} to our case. (Strictly speaking, for complexity reasons, we consider a ``local" variant of this method as described in the neighborhood methods below.)\\
\noindent {\bf Local Regression (LoR):} Neighborhood methods have proved their worth in collaborative filtering as a good scalable approach \cite{Koren3, barman_2, citeulike:4529451}. Motivated by such methods, we consider local regression  for prediction, where the ``local" neighborhood of an experiment is identified from individual estimates of the regression vectors. 

In Section \ref{sec:empirical_results}, we report the prediction performance of the various algorithms on the YLRC dataset as well as simulated data. We compare their implementation complexities, runtimes, mean-square error (MSE) and classification error (CE). In particular, we identify the LoR method as a good choice that yields near-best performance with acceptable runtime. The LoR method is also less sensitive to the choice of algorithm parameters. Motivated by this finding, we 
analyze the LoR method for a mathematical model with $K_0 \leq M$ clusters of experiments and Gaussian prediction errors. Our analysis exploits known results on the asymptotic normality of maximum likelihood signal parameter estimates \cite{Poor1}. We find that the optimal neighborhood size of the LoR method coincides with the true cluster size, and this gives insight into the query cluster sizes in the YLRC dataset.   We also study performance of the LoR method for different noise levels. For the entire range of noise values, the LoR method improves over individual as well as collective regression. At small noise levels, the LoR method finds the correct neighborhood, and as expected, the MSE is smaller than that for individual regression by a factor equal to the cluster size. As the noise level increases, its performance degrades gracefully and approaches that of collective regression.
The conclusion is given in Section \ref{sec:conclusion} and in the Appendices we fill in several details.

\noindent
{\bf Notation:}
Matrices and vectors are written in bold upper and bold lower case
letters respectively. All vectors are column vectors and $[.]^T$ denotes the transpose. By $\mathcal N(.|\mu, \sigma^2)$, we denote the Gaussian density with mean $\mu$ and variance $\sigma^2$, i.e., $\mathcal N(x|\mu, \sigma^2) = \frac{1}{\sqrt{2\pi}\sigma} e^{-\frac{(x-\mu)^2}{2\sigma^2}}$.

\section{Basic Setup}
\label{sec:model}

Consider $M$ experiments and let ${\mathbf x_{mn}} \in \mathbb R^d$, $y_{mn}\in \mathbb R$, $n=1,...,N_m$ be the prediction and response variables respectively for experiment $m$, $1 \leq m \leq M$. In general, we would like to consider regression of the form:
\[
y_{mn} = g_m({\mathbf x_{mn}}) + w_{mn}
\]
where $w_{mn}$ is the prediction error. If we expect different experiments to have similar regression functions $g_m(\cdot)$, then by pooling their data together, we hope to be able to estimate $g_m$ better, and hence obtain improved prediction. If the response and predictor variables across experiments exhibit clustering, then it is conceptually easy to identify the experiment clusters. 
In this paper, our interest is in the case where the regression functions exhibit clustering, even though the response and predictor variables may not show any clustering.\footnote{However, we note that our algorithms work even if the data itself exhibits clustering.} Our goal is to study mechanisms for pooling data from different experiments to improve the estimation of the regression function and we focus on the special case of linear regression: 
\begin{align}
\label{eq:model}
y_{mn} = \mathbf h_m^T \mathbf x_{mn} + w_{mn}.
\end{align}
In Section \ref{sec:approaches}, we suggest several methods to pool data from across experiments to improve estimation of $h_m$ and consequently improve prediction performance. 

Is the above viewpoint useful in any applications? 
Our empirical results in Section \ref{sec:empirical_results} suggests that the YLRC dataset benefits from this viewpoint. For this dataset,
 each experiment corresponds to a web search with a given
{\it query} and trials of an experiment correspond to the different URLs in the search
results. Each predictor variable is a  feature vector of the corresponding
query-URL pair, and the corresponding response variable is a score
indicating relevance of that URL to the query. Relevance scores are in
the range of 0 (irrelevant) to 4 (perfectly relevant). We expect that many queries
are related to each other and hence pooling their data together may improve prediction. The challenge lies in the fact that, apriori, we do not know which queries are similar, and the feature vectors/relevance scores do not  exhibit clustering on their own. 

To gain further insight, we also test our methods on simulated data. For simulating the data as well as the mathematical analysis, the following  model is useful. Suppose that the regression vectors $\mathbf h_m$
are chosen uniformly randomly from a set of  $K_0$ distinct vectors, and suppose
that $\{w_{mn}\}$ are independent Gaussian with mean 0 and variance $\sigma^2$. For this model, we have $K_0$ unknown clusters of experiments. Since we are interested in the extreme case where the data across experiments is diverse and does not exhibit any obvious clustering, we consider independent predictor variables across experiments. In particular, we choose $\{\mathbf x_{mn}\}$ to be i.i.d.\ Gaussian with zero mean and identity covariance matrix. The $K_0$ true regression vectors are generated randomly, uniformly on the unit sphere in $\mathbb R^d$.

\section{Algorithms for CRUC}
\label{sec:approaches}
In this section, we briefly discuss various approaches to pool data from different experiments.

\subsection{Individual Regressions (IR)}
\label{sec:individual}
For every experiment $m$, we fit a different least square regression on the
training data. Let $\mathbf X_m := [\mathbf x_{m1}^T, ..., \mathbf
x_{mN_m}^T]^T$ denote the matrix of predictor variables for the $m$th
experiments, while the corresponding vector of response variables is denoted by  $\mathbf
y_m:=[y_{m1}, ..., y_{mN_m}]^T$.  Then the 
least-square estimate for the regression  vector is \cite{bishop1}\footnote{We assume throughout that $N_m \geq d$ and the data is well-conditioned so that the desired matrix inverses exist.}
\[ \hat{\mathbf h}_m = (\mathbf X_m^T \mathbf X_m)^{-1}\mathbf X_m^T \mathbf y_m.\]

\subsection{Collective Regression (CR)}
Instead of having a different linear model for each experiment, we can fit a
common linear regression vector for all the experiments. Let $\mathbf X:=[\mathbf
X_1^T, ..., \mathbf X_M^T]^T$ and $\mathbf y:=[\mathbf y_1^T, ...,
\mathbf y_M^T]$ denote the set of predictor variables and the response variables for the whole data respectively. Then the least square estimate
is 
\[ \hat{\mathbf h} = (\mathbf X^T \mathbf X)^{-1} \mathbf X^T \mathbf y.\]

\subsection{An EM Algorithm (EM)}

In order to motivate this approach, recall the mathematical model with Gaussian noise introduced towards the end of Section \ref{sec:model}. Suppose we have $K\leq M$ distinct regression vectors and let $p_k$ denote the probability
that the $k$th regression vector is chosen for an experiment. Since we don't know the clusters, the cluster index of experiments is not known. Treating this as the missing data, we get a link to the EM setup in \cite{dempster}. Using this link, in Appendix \ref{app:em} we derive an EM algorithm that iteratively computes estimates of class probabilities $p_k$, noise variance $\sigma^2$, and the true regression vectors. The steps in this algorithm are described in the following.

\begin{center}
\framebox{\parbox{3.2in}{
{\bf Iterations of the EM algorithm:}

- Require: an integer $K, 1\le K\le M$; a guess for the number of clusters.

- Initialize:  $\mathbf h_{k,(0)}, p_{k,(0)}, k=1,2,...,K$ and $\sigma_{(0)}$; initial guess for the $K$ regression vectors, their assignment probabilities and the noise variance respectively.

- For the $t$-th iteration:

{\bf E step}: 
For $m=1,2,...,M$, $k=1, 2, ..., K$, 
\[ \gamma_{mk,(t)}= \frac{p_{k,(t-1)}\prod_n \mathcal N(y_{mn}|\mathbf
  h_{k,(t-1)}^T \mathbf x_{mn}, \sigma_{(t-1)}^2)}{\sum_k p_{k,(t-1)}\prod_n
  \mathcal N(y_{mn}|\mathbf h_{k,(t-1)}^T \mathbf x_{mn},
  \sigma_{(t-1)}^2)}.\]

{\bf M step}: 
For {$k=1, 2, ..., K$}
\[ \mathbf h_{k,(t)} = \mathbf A_{k,(t)}^{-1} \mathbf b_{k,(t)}, \ \ p_{k,(t)} = \frac{\sum_m \gamma_{mk,
    (t)}}{\sum_{m,k}\gamma_{mk, (t)}} ,\]
where $\mathbf A_k$ and $\mathbf b_k$ are as in \eqref{eq:h} in
Appendix \ref{app:em},

and 
\[  \sigma_{(t)}^2 = \frac{\sum_{m,n,k} \gamma_{mk,(t)}
  (y_{mn}-\mathbf h_{k,(t)}^T \mathbf x_{mn})^2}{\sum_{m,n,k}
  \gamma_{mk,(t)}}.\]

}}
\end{center}

\subsection{A K-means Algorithm (KM)}
The mathematical model considered for motivating the EM algorithm above also provides a basis for this algorithm, which is reminiscent of the classical K-means clustering algorithm \cite[Section 9.1]{bishop1}.
We start with an initial list of $K$ regression vectors. In each
iteration of the algorithm, for each experiment, we find the regression vector from the list of $K$ regression vectors that leads to the least MSE. This leads to a clustering of the experiments. We pool together data from all the experiments in a cluster and find the least-squares regression vector. This yields a new list of $K$ regression vectors and the method continues. 


\begin{center}
\framebox{\parbox{3.2in}{
{\bf Iterations of the K-means algorithm:}

- Require: an integer $K, 1\le K\le M$; a guess for the number of clusters.

- Initialize: $\mathbf h_1^{(0)}, \mathbf h_2^{(0)}, ..., \mathbf h_K^{(0)}$; initial guess for the $K$ regression vectors.

- For the $t$-th iteration: 

For $m=1,2,...,M$, 
\[k_m :=\arg \min_k \sum_n (y_{mn}-\mathbf x_{mn}^T
\mathbf h_k^{(t)})^2.\]

For  $k=1,2,..., M$, find least square estimate $\hat h_k^{(t+1)}$ using data from the experiments with $k_m = k$. 
}}
\end{center}


\begin{table*}
  \centering
  \begin{tabular}{|c|c|c|c|c|c|c|c|c|c|}
\hline
& \multicolumn{3}{|c|}{\bf MSE} & \multicolumn{3}{|c|}{\bf Classification
  error} & \multicolumn{3}{|c|}{\bf Runtime (sec)}\\
\hline
{\bf Method} & \texttt{SMALL} & \texttt{MEDIUM} & \texttt{LARGE} &
\texttt{SMALL} & \texttt{MEDIUM} & \texttt{LARGE} & \texttt{SMALL} & \texttt{MEDIUM} & \texttt{LARGE}\\
    \hline
IR &.7312 &.9494 &1.7746 &.0619 & .1097&.1628 & 1.5322&17.581 &145.2082\\
\hline
CR &.7351 &.7865  &.80 &.0536 & .0846&.0967 &1.4213 &16.9712 &134.2348\\
\hline
EM &.6244 &.6924 &-NA- &.0530 &.0840 &-NA- &507.6 &18880 &-NA-\\
\hline
KM &.6227 & .6965  &-NA- &.0522 &.0842 &-NA- &88.84 &421.2 &-NA-\\
\hline
SVT &.5896 & -NA-\ &-NA- &.0510 &-NA- &-NA- &15.8 & -NA-& -NA-\\
\hline
CW &.6762 & .7270& 0.7399 &.0525 &.0841 & .0968 &1.8566 &24.609 & 841.717\\
\hline
 LoR &.6229 & .7012  &.7291 &.0520 &.0836 &.0982 &1.6042 & 22.961&834.597\\
\hline
 \end{tabular}
  \caption{Error and Runtime comparison of various methods on real data}
  \label{table:mse}
\end{table*}

\subsection{A Singular Value Thresholding Algorithm (SVT)}
If we have $K_0 < d$ clusters as in our mathematical model, then the matrix $\mathbf H:=[\mathbf h_1, ...,
\mathbf h_M]$ of the regression vectors has low rank. Motivated by this, consider the following optimization problem. For a
given matrix $\mathbf G$, let $\text{MSE}(\mathbf G)$ denote the resulting MSE: 
\[\text{MSE}(\mathbf G) = \sum_m \sum_n (y_{mn} - \mathbf g_m^T \mathbf x_{mn})^2,\]
where $\mathbf g_m$ is the $m$th column of $\mathbf G$.
The low-rank assumption suggests that for an $\epsilon > 0$, we should solve
\begin{align}
\label{eq:opt}\textrm{minimize } \mathtt{rank}(\mathbf G),
\textrm{such that } \text{MSE}(\mathbf G) \le \epsilon.
\end{align}
This is a rank minimization problem with a quadratic constraint, and we
use the singular value thresholding (SVT) approach of \cite[Section 3.3.2]{Candes_cai}
to solve it. (In Appendix \ref{app:svt}, we provide more details how \eqref{eq:opt} falls under the formulation of \cite{Candes_cai}.)

\subsection{The Curds and Whey Method (CW)}
To utilize correlation among response variables, \cite{breiman1}
introduced a method called Curds and Whey that finds an
``optimal'' linear combination of the individual regression estimates.
We use a ``local'' version of this method here, which we describe next. The method takes an integer parameter $T \leq M$ and starts by forming individual least-squares regression vectors. For any two experiments, we think of the Euclidean distance between their estimates as the distance between the experiments. For each experiment, we then form a list of $T$ closest experiments (including the experiment itself), and this set is denoted by $\mathcal T$. 
Consider now the predictor vector $\mathbf x_{mn}$ and suppose the individual estimates are $\hat y_{mn}^{(t)}:=\hat{\mathbf h}_t^T \mathbf x_{mn}$, $t=1,2,..., T$, where $\hat{\mathbf h}_t$ is the $t$-th individual estimate from the list of closest experiments. To obtain the final estimate, we consider a linear estimate of the form 
\[\tilde{ y}_{mn} = \sum_{t\in \mathcal T} \beta_t \hat{\mathbf h}_t^T \mathbf x_{mn}.\]
The parameters $\beta_t$ are chosen by minimizing the MSE for each experiment $m$ separately.  

We note that the CW method fits a least square regression on the data in $\mathcal T$, restricted to the span of $\{\hat{\mathbf h}_t\}_{t\in \mathcal T}$. Suppose $\mathcal D_T$ denotes the indices of all the trials  restricted to the experiments  in $\mathcal T$. Then the CW estimate for the $m$th regression vector is 
\[\tilde{\mathbf h}_m = \arg \min_{\mathbf h\in \mathrm{span}(\{\hat{\mathbf h}_t\}_{t\in \mathcal T})} \sum_{i\in \mathcal D_T}(y_i -\mathbf h^T \mathbf x_i)^2.\]
Suppose $\mathbf H_{\mathcal T}$ denotes the $d\times T$ matrix whose columns are the individual estimates of the $T$ most similar experiments. Let $\mathbf X_{m,\mathcal T}$ and $\mathbf y_{m,\mathcal T}$ denote the predictor variables and the corresponding response variables, restricted to the experiments in $\mathcal T$. For a predictor variable $\mathbf x_i \in \mathbf X_{m,\mathcal T}$, let $\mathbf z_i := \mathbf x_i^T \mathbf H_{\mathcal T}$, and suppose $\mathbf Z_{m,\mathcal T}$ denotes the matrix whose rows consists of these vectors $\mathbf z_i$'s. Then the least square estimate for the $\boldsymbol \beta$-vector for experiment $m$ is 
\[\hat {\boldsymbol \beta}= (\mathbf Z_{m,\mathcal T}^T \mathbf Z_{m,\mathcal T})^{-1}\mathbf Z_{m,\mathcal T}^T \mathbf y_{m,\mathcal T}, \]
and the corresponding estimate for the regression vector is, 
\[\tilde{\mathbf h}_m = \sum_t \hat \beta_t \hat{\mathbf h}_t.\]

\subsection{Local Regression (LoR)}
\label{subsec:lr}
In this approach, similar to the CW method, for every experiment, we first find the $T$ most similar experiments. 
Let $\mathbf X_{m,\mathcal T}$ and $\mathbf y_{m,\mathcal T}$ denote the predictor vectors and the response scores restricted to these $T$ nearest experiments. Then the final estimate for the $m$th regression vector is 
\[\tilde{\mathbf h}_m = (\mathbf X_{m,\mathcal T}^T \mathbf X_{m,\mathcal T})^{-1} \mathbf X_{m, \mathcal T}^T \mathbf y_{m,\mathcal T}.\]

\begin{remark}
\label{remark:LR_CW} 
Let $\mathbf H_{\mathcal T}$ denote the $d\times T$ matrix of the $T$ individual
estimates closest to an experiment. When $\mathrm{span}(\mathbf H_{\mathcal T}) = \mathbb R^d$, i.e., the
individual estimates span the whole space,  then the
CW method is same as the LoR method. Otherwise, the CW
method finds a set of estimates that are restricted to a subspace
(linear span of the individual estimates).
\end{remark}

\begin{remark}
\label{remark:clustering}
While the EM and KM algorithms perform joint regression and clustering, the other methods do not perform any explicit clustering.
\end{remark}

\section{Performance Comparison}
\label{sec:empirical_results}
In this section, we compare the different methods on various
samples of the YLRC dataset as well as on data simulated as per the
Gaussian model. Along with MSE and CE, we also discuss the
algorithms in terms of their runtime and complexity.

\subsection{Evaluation on YLRC Dataset}
\begin{figure*}
  \centering
  \subfigure[sensitivity]{
  \includegraphics[width=3.0in]{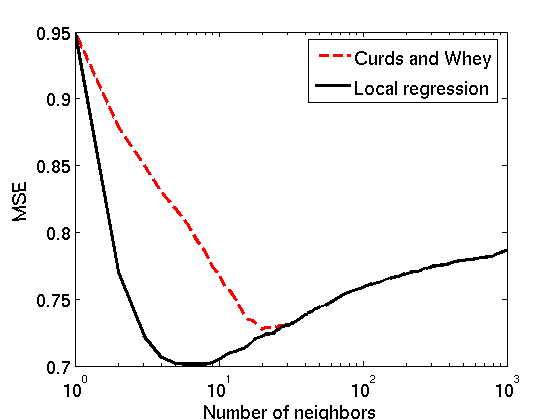}
  \label{subfig:sensitivity_kernel_CW}
}
 \subfigure[running time]{
  \includegraphics[width=3.0in]{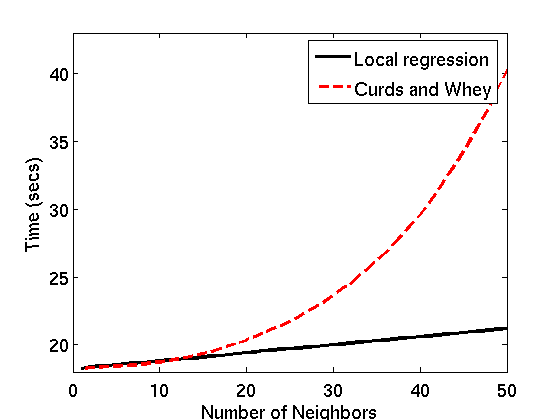}
  \label{subfig:time_kernel_CW}
}
  \caption{Sensitivity and running time comparison of the LoR and CW methods as number of neighbors $T$ varies for the \texttt{MEDIUM} dataset}
\end{figure*}

The feature vectors of the original YLRC dataset are 700 dimensional
and sparse. Motivated by compressed sensing of sparse vectors using random projections \cite{Wang:EECS-2009-169, citeulike:4109397}, for
computational tractability, we project the feature vectors to a
randomly chosen 20 dimensional space, that is, each feature vector is replaced by a 20 dimensional sketch. We work only on a subset of
the queries available in the dataset. The \texttt{LARGE}
dataset consists of 5000 queries, the \texttt{MEDIUM} dataset
consists of 1000 queries, and the \texttt{SMALL} dataset consists of
100 queries.

We split each of these datasets into  training (70\%) and a test dataset
(30\%) randomly. Given the training data and the feature vectors of
the test data, we want to predict the corresponding relevance scores. 
We compare the methods on the basis of their MSE and CE on the test
datasets. To
calculate the CE, we classify the scores as \texttt{HIGH} (3,4) or
\texttt{LOW} (0,1,2), and count the fraction of times when a
\texttt{HIGH} score is estimated as \texttt{LOW} or vice versa. We
also compare all the methods based on the runtimes of their respective
Matlab implementations on an Intel Xeon 4-core 2.67 GHz machine with
16GB of RAM.
Some of the methods considered have an input parameter (e.g., number of
neighbors in the LoR and CW methods; number of clusters in the EM and KM algorithms). For comparison, for each algorithm, we choose the value of the input parameter that yields the least MSE. The MSE and the CE are averaged over 10 realizations of the random splitting into training and test datasets.

\noindent
{\bf Error comparison:} 
Table \ref{table:mse} summarizes the performance of all the algorithms.  For the \texttt{SMALL} dataset, we see
that the SVT algorithm performs the
best, while the LoR, EM and KM methods are not too far behind. But for the
\texttt{MEDIUM} dataset, we could not run the SVT algorithm in our setup
due to large memory requirements. For this dataset CR performs better
than individual regressions, and the EM algorithm shows the best performance, while
the LoR method is quite close. If we observe the running times of the
algorithms, the EM and KM algorithms are very slow compared to the LoR and CW methods. In fact, for
the \texttt{LARGE} dataset,  the EM and KM algorithms take too long and we
could not report their MSE. For this dataset, CR gives more than
100\% improvement over IR, and the LoR method gives further 10\% improvement. The CW
method also shows similar (slightly worse) performance. Hence we see that the LoR
 method has  a great advantage over global algorithms like the EM, KM and
SVT algorithms in terms of the runtime, and also yields competitive performance
in terms of MSE as well as CE.

\noindent
{\bf Predictor variables are not clustered:}
The CRUC framework and the associated algorithms are most relevant when the data does not exhibit clustering but the regression vectors are clustered. (However, we note that all the above algorithms work even when the data is clustered.) 
The YLRC dataset appears to be an example of such a situation. To demonstrate this, we consider an arbitrary query, and let  $\mathbf A$ denote the matrix of feature vectors restricted to the 10 most similar queries. The second largest eigenvalue of the row-normalized correlation matrix of $\mathbf A$ is never more that 0.23 \footnote{From the spectral clustering literature \cite{Dhillon_1}, we know that number of eigenvalues close to unity is an estimate of the number of clusters.}, which suggests that the predictor variables are not clustered. But the regression vectors across queries show clustered behavior and hence the various methods proposed show substantial improvement over IR. Moreover, these methods also outperform CR, which indicates that there are several query clusters.

\noindent
{\bf Impact of algorithm parameters:}
The LoR and CW methods take as input a parameter $T$ for the neighborhood size, while the KM and  EM algorithms need a parameter $K$ representing the cluster size. We next study the sensitivity of the MSE to the choice of these parameters. 

In Figure \ref{subfig:sensitivity_kernel_CW}, we compare the MSE of the LoR and CW methods as we vary $T$. We see that when $T$ is
large, performance of the CW method matches with the LoR method, which is consistent with
Remark \ref{remark:LR_CW} in Section \ref{subsec:lr}.
From the experiments on the simulated data in Section \ref{sec:simulated}, we shall see that in a high SNR regime, the optimal neighborhood size is very close to the true cluster size. Figure \ref{subfig:sensitivity_kernel_CW} thus suggests that the regression vectors are clustered, and the effective cluster size for the \texttt{MEDIUM} dataset is close to 9.
Figure  \ref{subfig:time_kernel_CW} shows that the running time of the LoR method increases linearly with $T$, whereas 
the CW method shows a super-linear growth. (In Section \ref{sect:complexity}, we show that the runtime of the CW method is $O(T^3)$.) Thus while both the LoR and CW methods yield near best MSE and CE, the runtime of the LoR method is smaller.

Figure \ref{subfig:sensitivity_EM_kM} shows that the performance of the KM algorithm is less sensitive to its parameters compared to the EM algorithm. However, both these methods are more sensitive to under-estimation of the optimal parameter compared with the LoR and CW methods. From Figure  \ref{subfig:sensitivity_EM_kM}, we see that runtime the KM algorithm scales better than the EM algorithm, but both are worse compared with the LoR and CW methods.

\begin{figure*}
\centering
  \subfigure[sensitivity]{
  \includegraphics[width=3.0in]{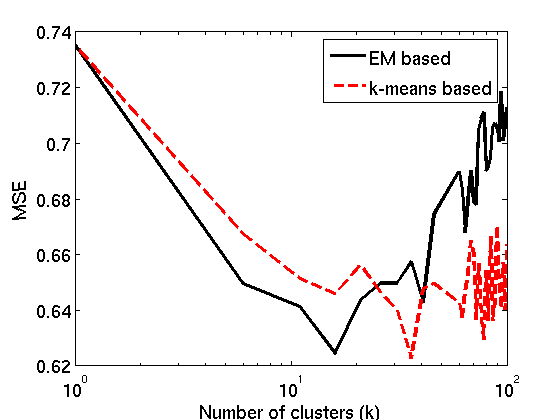}
  \label{subfig:sensitivity_EM_kM}
}
 \subfigure[running time]{
  \includegraphics[width=3.0in]{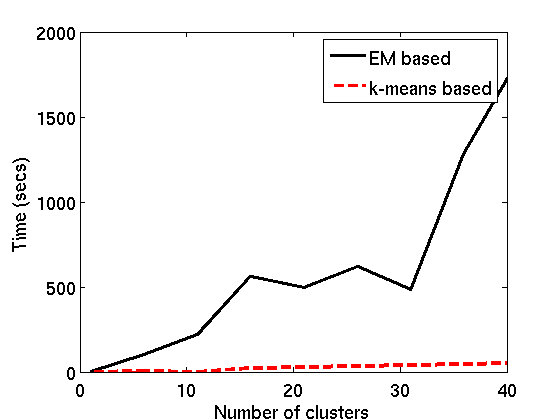}
  \label{subfig:time_EM_kM}
}
  \caption{Sensitivity and running time comparison of the EM and KM algorithms as the cluster size $K$ varies for  the \texttt{SMALL} dataset}
  \label{fig:sensitivity}
\end{figure*}

To understand the above runtimes, we further analyze the
computational complexity of all the different methods in the next section. 
\subsection{Complexity Comparison}
\label{sect:complexity}
To avoid cluttered expressions, in this section we assume there are
same number of trials in each experiment, and we denote this number by $N$.
We also assume that the dimension parameter $d$ is of constant order.

\noindent 
{\bf IR:}
To compute the $m$th estimate, we need
$O(N)$ operations to 
multiply a $d\times  N$ matrix $\mathbf X_m^T$ with its transpose;
constant number of operations to invert a $d\times d$ matrix $\mathbf
X_m^T\mathbf X_m$; $O(N)$ operations to multiply this $d\times d$
inverse with another $d\times N$ matrix $\mathbf X_m^T$; and $O(N)$
operations to multiple this resulting $d\times N$ matrix by a vector
$\mathbf y_m$. Thus we need $O(N)$ operations for each of the
estimates, requiring $O(MN)$ operations in total.

\noindent
{\bf CR:} 
By a similar analysis as above, we see that  we need $O(MN)$
operations to do a complete regression.

\noindent
{\bf EM:}
In the E-step of each iteration, we need $O(MNK^2)$ operations.
In the M-step, we need $O(MNK)$ operations to construct the matrix
$\mathbf A$ and the vector $b$, and, a constant number of operations to compute $\mathbf
A^{-1}\mathbf b$. To compute the variance, we need $O(MNK)$
operations. Thus for the EM algorithm, we require $O(MNK^2)$
operations per iterations.

\noindent
{\bf KM:}
 At each iteration, starting with a list of $K$ regression vectors, for each experiment, we need $O(NK)$
operations to obtain the regression vector that best explains the
data. Thus we require $O(MNK)$ operations for clustering the experiments. After
grouping the experiments based on the closest regression vector,
suppose there are $m_k$ experiments corresponding to the $k$th
vector. To do the corresponding regression, we need $O(m_k N)$ operations, resulting in a total of $O(MN)$
operations for all the $K$ regressions. Thus we need $O(MNK)$
operations for each iteration of the KM algorithm.

\noindent
{\bf SVT:}
In every iteration of this method, we need to multiply a $Md\times MN$
matrix by a $MN\times 1$ vector, requiring $O(M^2N)$ operations, and
then to perform the singular value thresholding on a $M\times d$ matrix, we need $O(M)$
operations. Thus in each iteration of this algorithm, we need
$O(M^2N)$ operations.

\noindent
{\bf CW:}
After we have performed all the individual regressions for each
experiment, we need to compute the $T$ most similar experiments. To do this we
need to compute similarities with all the experiments, requiring $O(M)$
operations; and to find the $T$ most similar experiments require $O(TM)$
operations.
Then to compute the
matrix $\mathbf Z_{m,\mathcal T}$ we need $O(NT^2)$ operations; and
to compute the regression using this matrix requires $O(T^3N)$
operations. Thus to compute the estimates for all the experiments, we need
$O(M^2T+ MNT^3)$ operations, including
the operations needed to compute the individual estimates.

\noindent
{\bf LoR:}
As in the CW method, we need $O(MT)$ operations  to find the $T$ most similar experiments. Then to perform a regression on the data of these $T$
experiments require $O(TN)$ operations, requiring a total of $O(MNT+ M^2T)$ operations for all the $M$ estimates.

\begin{table}
\centering
\label{table:complexity}
\begin{tabular}{|c|c|}
 \hline
{\bf Method} & {\bf Complexity}\\
\hline
IR & $O(MN)$\\
\hline
CR & $O(MN)$\\
\hline
EM (per iteration) & $O(MNK^2)$\\
\hline
KM (per iteration) & $O(MNK)$  \\
\hline
SVT (per iteration) & $O(M^2N)$\\
\hline
CW & $O(M^2T + MNT^3)$ \\
\hline
LoR &   $O(M^2T + MNT)$\\
\hline
\end{tabular}
\caption{Complexity comparison of methods}
\end{table}

We summarize the complexity of these algorithms in the
Table \ref{table:complexity}. We see that the LoR method has a linear growth with
$T$, unlike the CW method that has a cubic growth. This is consistent with the
empirical evaluation as seen in Figure
\ref{subfig:time_kernel_CW}. Similarly, the KM algorithm shows a linear growth with
$K$ whereas the EM algorithm shows a quadratic growth, which is consistent with the
observation in Figure \ref{subfig:time_EM_kM}. Table
\ref{table:complexity} also shows that the EM and KM algorithms have a linear growth
with $M$, whereas the LoR and CW methods has quadratic growth. This is unlike what
we see on real data, as shown in Table \ref{table:mse}, where the run
times of the  EM and KM algorithms scale much worse with $M$ compared to the LoR and
CW methods. This could be due to bad constant factors for the EM and KM algorithms. IR and CR
has a linear growth with $M$, as is consistent with performance on
real data shown is Table \ref{table:mse}.


Based on the above discussions, we see that the LoR method is attractive from several viewpoints: it provides near best MSE and CE performance at reasonable computational load and is also less sensitive to the choice of the input parameter than other methods. To understand it better, in the next section, we consider simulations and some mathematical analysis of the LoR method. 

\subsection{Evaluation on Simulated Data}
\label{sec:simulated}
Recall the mathematical model for the data with Gaussian noise as introduced towards the end of Section \ref{sec:model}. We use $M=100$ experiments, $d=6$ feature vector dimension, $K_0=5$ number of clusters, and $N=70$ trials per experiment. We average the MSE over several realizations of the true regression vectors and the data.

In Figure \ref{subfig:mse_gain}, we see how IR, CR, the LoR method and the EM algorithm perform
as noise variance changes.  When the noise level is low, 
the LoR method can find the right neighbors and hence collaboration helps. The LoR method
shows better MSE than IR and CR, and for small noise levels, as expected its MSE is less than IR by a factor equal to the cluster size. As the noise level increases, it
becomes harder to find the right neighbors, and the optimal LoR method
picks all the neighbors to perform the regression and performs as good as CR. This phenomenon is
more clear in Figure
\ref{subfig:optimal_T}, which shows that when the noise level is low, the
optimal value of $T$ (the neighborhood size) is close to the actual
cluster size, and as the noise level increases the
optimal $T$ is roughly same as the total number of experiments. In other words, for high noise levels, the best strategy is to filter out noise by pooling data from all the experiments. In between the two extremes of high and low noise, the LoR method provides a graceful transition by balancing the need to average out noise and the need to better estimate the regression vector.
\begin{figure*}
  \centering
  \subfigure[MSE plot]{
\includegraphics[width=3.0in]{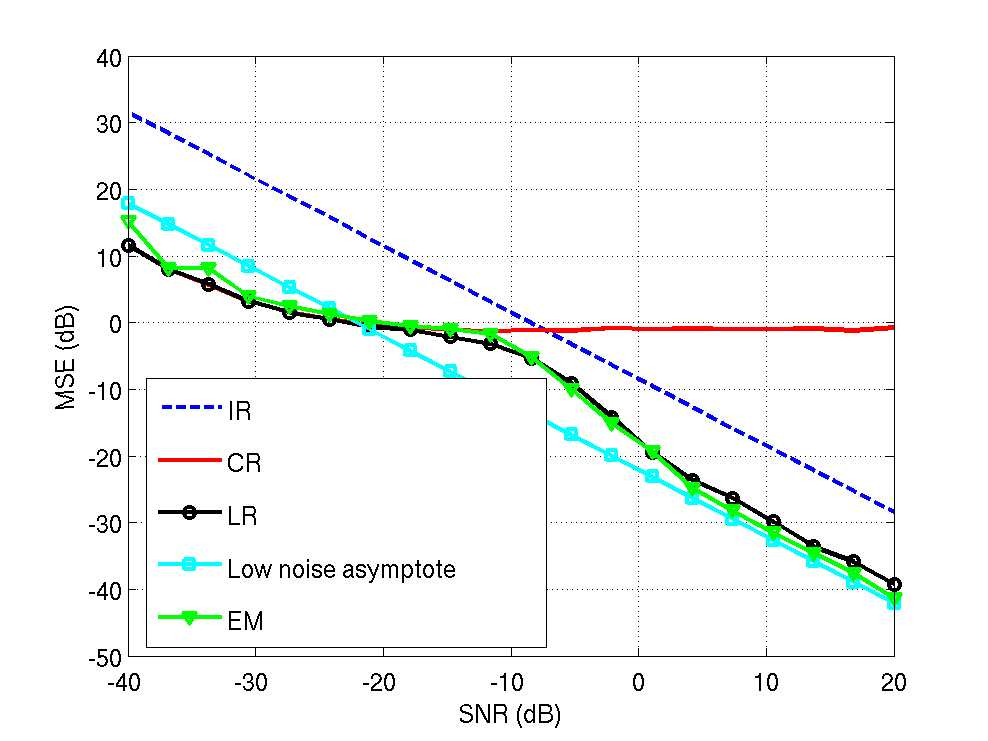}
\label{subfig:mse_gain}
}
\subfigure[Optimal neighborhood size]{
\includegraphics[width=3.0in]{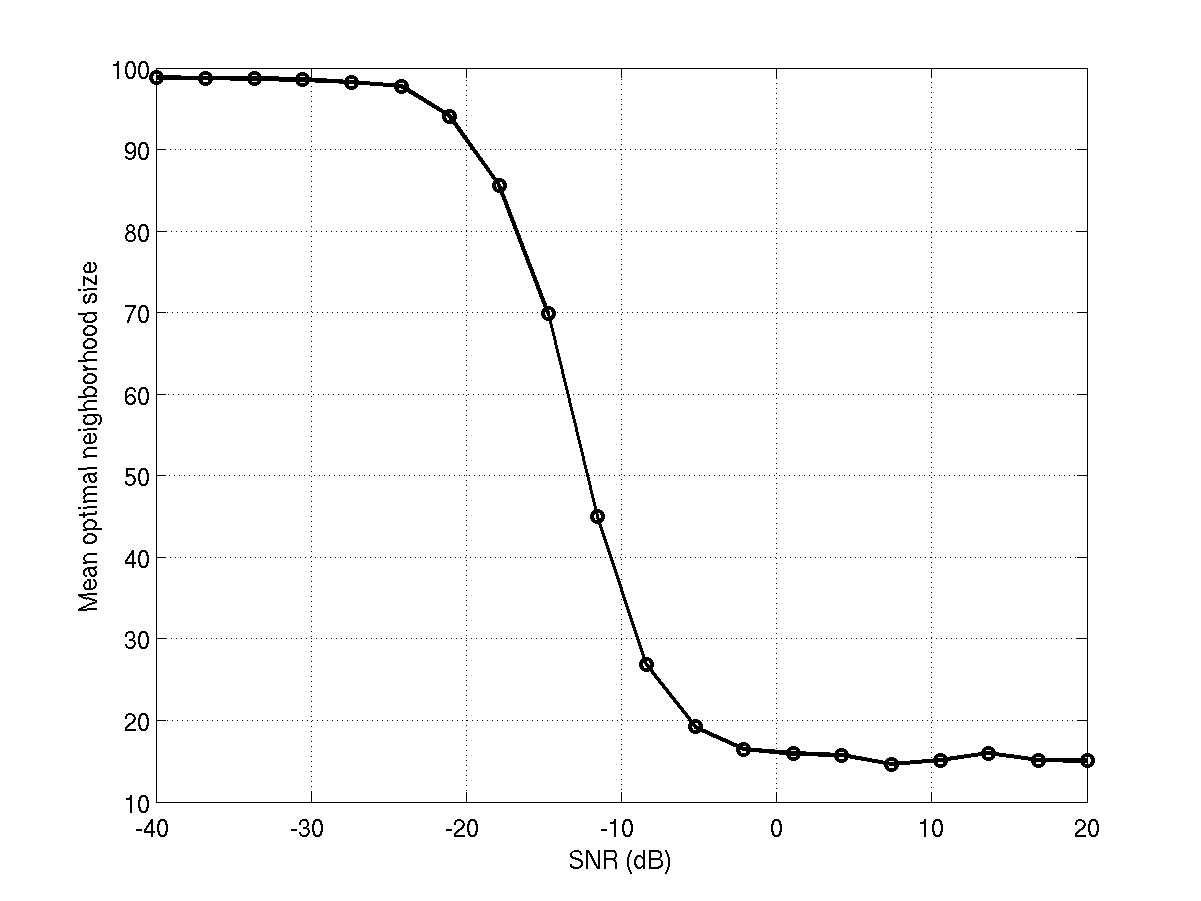}
\label{subfig:optimal_T}
}
  \caption{Performance comparison as the noise variance changes (for
    $M=100, N=70, K_0=5, d=6$)}
  \label{fig:performance_with_noise}
\end{figure*}

We also plot the performance of the EM algorithm as the noise level changes. Since the EM iterations can converge to a local maxima, we estimate the outage probability (probability that the iterations converge to a wrong optimum), and compute its MSE only over the trials that converge to something close to the true regression vectors. Figure \ref{subfig:mse_gain} shows that the EM algorithm  and the  LoR method has a very similar performance for all noise levels, while we observe an outage probability of 9.7\%, averaged over all noise levels.

\noindent
{\bf High SNR asymptote:}We assume that in the high SNR regime, the optimal LoR method picks
all the right neighbors to perform the local regression, and in this case, the estimate corresponds to a maximum likelihood (ML) estimate. Hence to find the high SNR asymptote, we can use standard results about ML estimate. If there are
$T$ neighbors, each with $N$ data points, then we have a local
regression with $TN$ data points for each experiment. Let $\hat{\mathbf
  h}_m$ be the estimated regression vector for $m$th experiment. We have the
following result.
\begin{proposition}
  \label{thm:high_snr}
For each $m=1,2,..., M$,
  \[ \sqrt{TN} (\hat{\mathbf h}_m -\mathbf h_m) \rightarrow \mathcal
  N(0,\sigma^2 \mathbf I),\]
where $\mathbf I$ is the $d\times d$ identity matrix, $\mathcal N(\mu, \Sigma)$ denotes a Gaussian distribution with mean $\mu$ and covariance matrix $\Sigma$, and,  the
convergence is in distribution.
\end{proposition}
\begin{proof}
Let
$Pr[y_{mn}]$ denote the likelihood of a
trial. Then the $(i,j)$th entry of the Fisher information matrix (see
\cite[Section IV.E.1]{Poor1}) is defined as 
\begin{align*}
(\mathbf I_F)_{i,j} &:= \mathbb E_{\mathbf h_m} \left[\left(\frac{\partial}{\partial h_{mi}} \log
Pr[y_{mn}]\right) \left(\frac{\partial}{\partial h_{mj}} \log
Pr[y_{mn}]\right) \right]\\
& = \frac{1}{\sigma^4}\mathbb E_{\mathbf h_m} \left[ (y_{mn}-\mathbf
  h_m^T \mathbf x_{mn})^2 \mathbf x_{mn}(i) \mathbf x_{mn}(j)\right]\\
& = \frac{1}{\sigma^4}\mathbb E_{\mathbf h_m} \left[ w_{mn}^2\right]\mathbb E_{\mathbf h_m}\left[\mathbf
  x_{mn}(i) \mathbf x_{mn}(j)\right]\\
& = \frac{1}{\sigma^2} \mathbf 1_{i=j},
\end{align*}
where the last equality follows since $w_{mn}\sim \mathbf
N(0,\sigma^2)$ and $\mathbf x_{mn}$ are chosen as i.i.d. unit normal.
Now standard asymptotic
normality of maximum likelihood estimators for vector parameter
estimation \cite[Section IV.E.1]{Poor1} tells us that
\begin{align*}
\sqrt(TN) (\hat{\mathbf h}_m - \mathbf h_m) \rightarrow \mathcal N(0,
\mathbf I_F^{-1}) = \mathcal N(0,\sigma^2 \mathbf I).
\end{align*}
\end{proof}
Proposition \ref{thm:high_snr} suggests that for a large enough $TN$,
$(\hat{\mathbf h}_m - \mathbf h_m)$ is almost normal with mean 0 and
variance equal to $\frac{\sigma^2}{TN}\mathbf I$.  Thus the mean MSE 
\[ \mathbf \mathbb E \left[\|\hat{\mathbf h}_m - \mathbf h_m\|^2
  \right] \approx \mathrm{trace}\left(\frac{\sigma^2}{TN}\mathbf I\right) = \frac{\sigma^2 d}{TN}.\]
In Figure \ref{subfig:mse_gain}, we compare this asymptotic MSE using $T=M/K_0$, with the
empirical performance on simulated data, and observe that the LoR method performs
almost as good as the asymptote even at SNRs up to 0 dB. 

\section{Conclusion}
\label{sec:conclusion}
We considered the CRUC framework, where experiments within a cluster show similar relationship between the predictor variables and the response variables. We introduced and studied various methods, and based on experiments on the Yahoo Learning-to-rank Challenge dataset and on simulated data, we observe that the local regression (LoR) method  is a good practical choice. It shows near optimal performance, scales reasonably well with data size, and is relatively insensitive to choice of algorithm parameter. We also analyze the LoR method  for an associated mathematical model that helps us to understand its prediction performance and optimal parameter choice. But we have only scratched the surface and we hope to present a deeper analysis of LoR in a future manuscript. A detailed study of the more general non-parametric CRUC case also appears to be a fruitful direction for further investigation.

\appendix

\section{Derivation of the EM algorithm}
\label{app:em}
Suppose $z_{mk}\in \{0,1\}$ is an indicator whether the $k$th regression vector $\mathbf g_k$ is
sampled for the $m$th experiment. Then  we have  $\sum_k z_{mk} = 1$.

A bit of notation before we begin: $\mathbf Y:=(\mathbf y_1,  \mathbf
y_2, ..., \mathbf y_M)^T$, and  $\mathbf G:=(\mathbf g_1, ..., \mathbf g_K)$. Similarly $\mathbf z_{m}:=(z_{m1},..., z_{mK})^T$, and $\mathbf Z:=(\mathbf z_1,..., \mathbf z_M)^T$.

Let $p_k$ be the probability of picking the $k$th regression
vector. Then $\sum_{k=1}^K p_k =1$, and suppose $\mathbf p:=(p_1,
...,p_K)^T$ represents the vector of probabilities. We can write the likelihood  as the following mixture.
\[ Pr[y_{mn}|\mathbf x_{mn}, \mathbf G] = \sum_{k=1}^K p_k\mathcal N(y_{mn}|\mathbf g_k^T \mathbf x_{mn}, \sigma^2), \]
and
\[ Pr[\mathbf Y|\mathbf X, \mathbf G] = \prod_m \sum_{k=1}^K p_k\prod_n \mathcal N(y_{mn}|\mathbf g_k^T \mathbf x_{mn}, \sigma^2).\]
The ML estimate finds an estimate that maximizes the above, i.e., 
\[ \mathbf G_{ML} = \arg \max_{\mathbf G,\sigma, \mathbf p} Pr[\mathbf Y|\mathbf X, \mathbf G] .\]
Because it is a mixture, it is not easy to maximize the above
directly. Instead we take an approach based on EM algorithms to
maximize the likelihood. In the EM based algorithm, we need to find
\cite[Section 9.3]{bishop1}, \cite{Borman2004a}
\[\arg\max_{\mathbf G, \sigma, \mathbf p} \mathbb E_{Pr[\mathbf Z|\mathbf Y,\mathbf X, \mathbf G]} \log Pr[\mathbf Y, \mathbf Z|\mathbf X, \mathbf G].\] 
In our case this can be computed.
We see that $Pr[z_{mk}=1] = p_k,$
since the $k$th regression vector is picked with probability
$p_k$, and thus
\[Pr [\mathbf z_m] = \prod_k p_k^{z_{mk}},  \]
and $Pr[y_{mn}|\mathbf x_{mn}, \mathbf G, z_{mk}=1] = \mathcal N(y_{mn}|\mathbf g_k^T \mathbf x_{mn}, \sigma^2). $
Thus we have
\[Pr[y_{mn}|\mathbf x_{mn}, \mathbf G, \mathbf z_m] = \prod_k \mathcal N(y_{mn}|\mathbf g_k^T \mathbf x_{mn}, \sigma^2)^{z_{mk}}. \]
This implies that 
\[Pr[\mathbf Y|\mathbf X, \mathbf G, \mathbf Z] = \prod_m \prod_k \left(\prod_n \mathcal N(y_{mn}|\mathbf g_k^T \mathbf x_{mn}, \sigma^2)\right)^{z_{mk}}, \]
and 
\begin{align}
  \label{eq:25}
 Pr[\mathbf Y , \mathbf Z|\mathbf X, \mathbf G] = \prod_m \prod_k \left( p_k\prod_n \mathcal N(y_{mn}|\mathbf g_k^T \mathbf x_{mn}, \sigma^2)\right)^{z_{mk}}.  
\end{align}
Suppose $\gamma_{mk}$ denotes the expected value of $z_{mk}$ w.r.t. the posterior distribution $Pr[\mathbf Z|\mathbf Y, \mathbf X, \mathbf G]$. Then we have 
\begin{align}
\nonumber   \mathbf \gamma_{mk} = \mathbb E_{Pr[\mathbf Z|\mathbf Y, \mathbf X, \mathbf G]} z_{mk}& = Pr[z_{mk}=1|\mathbf Y, \mathbf X, \mathbf G]\\
\nonumber & = \frac{Pr[\mathbf Y, z_{mk}=1| \mathbf X, \mathbf G]}{Pr[\mathbf Y| \mathbf X, \mathbf G]}\\
\nonumber & = \frac{p_k\prod_n \mathcal N(y_{mn}|\mathbf g_k^T \mathbf x_{mn}, \sigma^2)}{\sum_k p_k\prod_n \mathcal N(y_{mn}|\mathbf g_k^T \mathbf x_{mn}, \sigma^2)}.
\end{align}
Using the above and \eqref{eq:25}, we see that
\begin{align}
\nonumber S &:=  \mathbb E_{Pr[\mathbf Z|\mathbf Y, \mathbf X, \mathbf G]} \log Pr[\mathbf Y , \mathbf Z|\mathbf X, \mathbf G] \\
\nonumber  & =  \sum_m \sum_k \gamma_{mk} \left(\log p_k+ \sum_n \log \mathcal N(y_{mn}|\mathbf g_k^T \mathbf x_{mn}, \sigma^2)\right)\\
\nonumber & = \sum_m\sum_k \gamma_{mk} \log p_k - \\
\nonumber & \hspace{0.4in}\sum_{m,k,n} \gamma_{mk}\left(\log \sqrt{2\pi}\sigma + \frac{(y_{mn}-\mathbf g_k^T \mathbf x_{mn})^2}{2\sigma^2}\right).
\end{align}
To maximize w.r.t $\mathbf G$, we find that 
\begin{align}
\label{eq:h}
\nonumber & \frac{dS}{d\mathbf g_k}  = 0\\
\nonumber \Rightarrow   & \sum_m \sum_n  \gamma_{mk}(y_{mn}-\mathbf g_k^T \mathbf x_{mn}) \mathbf x_{mn} = 0\\
\nonumber \Rightarrow &  \underbrace{\sum_m \sum_n  \gamma_{mk} y_{mn} \mathbf x_{mn}}_{\mathbf b_k}  = \underbrace{\sum_m \sum_n  \gamma_{mk} \mathbf x_{mn} \mathbf x_{mn}^T}_{\mathbf A_k} \mathbf g_k\\
\Rightarrow &\mathbf g_k = \mathbf A_k^{-1} \mathbf b_k.
\end{align}
And to maximize w.r.t. $\sigma$, we obtain
\begin{align}
\label{eq:sigma}
  \nonumber &\frac{dS}{d\sigma} = 0 \\
\nonumber \Rightarrow &\sum_{m,n,k}  \gamma_{mk}\left(\frac{1}{\sigma} - \frac{(y_{mn}-\mathbf g_k^T \mathbf x_{mn})^2}{\sigma^3} \right)  = 0\\
\Rightarrow &\sigma^2 = \frac{\sum_{m,n,k} \gamma_{mk} (y_{mn}-\mathbf g_k^T \mathbf x_{mn})^2}{\sum_{m,n,k} \gamma_{mk}}.
\end{align}
To maximize w.r.t. to probability vector $\mathbf p$, we
have the following Lagrangian to be maximized.
\[L:= \sum_{m,k}\gamma_{mk} \log p_k + \lambda\left(\sum_k
    p_k -1\right).\]
By setting its derivative w.r.t. $p_k$ to zero, we obtain
\begin{align*}
\sum_{m} \gamma_{mk}/p_k +\lambda = 0, k=1,..., K.
\end{align*}
Using the above and the constraint that $\sum_k p_k=1$, we obtain
\begin{align}
\label{eq:p}
p_k = \frac{\sum_m \gamma_{mk}}{\sum_{m,k} \gamma_{mk}}, k=1,..., K.
\end{align}

Equation (\ref{eq:h}),  (\ref{eq:sigma}) and (\ref{eq:p}) completes
the derivation of the EM iterations.

\section{MSE as a Quadratic Constraint in SVT Algorithm}
\label{app:svt}
The optimization problem to be solved is to minimize rank of the  matrix with columns as the regression vectors, under an MSE constraint, i.e.,
\[\textrm{minimize } \mathtt{rank}(\mathbf G), \textrm{such that } MSE(\mathbf G) \le \epsilon.\]
Let  $\mathbf b:=(y_{11},...,y_{MN})^T$ denote the $MN\times 1$ vector
of all the relevance scores. And suppose $\mathbf A$ be the block
diagonal matrix whose $i$th  diagonal block ($i=1,..., M$)  is the matrix $\mathbf
X_i:=(\mathbf x_{i1}, ..., \mathbf x_{iN})^T\in \mathbb R^{N\times
  d}$. Also let $\mathbf g:=\mathrm{vec}(\mathbf G)\in \mathbb
R^{dM\times 1}$ denote the vector
obtained by stacking columns of $\mathbf G$ one after another. Suppose
$\mathcal A(\mathbf G):= \mathbf X\mathbf g$. $\mathcal A(.)$ is a
linear map from $\mathbb R^{dM}$ to $\mathbb R^{MN}$. Then we
have
\[MSE(\mathbf G) = \|\mathbf b- \mathcal A(\mathbf G)\|, \]
where $\|.\|$ denotes the Frobenius norm of a matrix. Thus the rank
minimization problem can be written as 
\[\textrm{minimize } \mathtt{rank}(\mathbf G), \textrm{such that }
\|\mathbf b- \mathcal A(\mathbf G)\| \le \epsilon,\]
which has the exact same formulation as the rank minimization problem
considered in \cite[Section 3.3.2]{Candes_cai}.

\bibliographystyle{abbrv}

\bibliography{myBib}  
%

\balancecolumns

\end{document}